\documentclass[preprint,journal]{vgtc}            

\vgtccategory{Data Transformations}

\title{When One Point Is Not Enough: Addressing Ambiguous Instances in Dimensionality Reduction by Splitting}

\author{%
  \authororcid{Diede P.M.\ van der Hoorn}{0000-0002-6348-1142},
  \authororcid{Alessio Arleo}{0000-0003-2008-3651}, and 
  \authororcid{Fernando V.\ Paulovich}{0000-0002-2316-760X}
}

\authorfooter{
  \item
  	Diede van der Hoorn, Alessio Arleo and Fernando Paulovich are with Eindhoven University of Technology.
  	E-mail: \{d.p.m.v.d.hoorn|a.arleo|f.paulovich\}@tue.nl.
}

\abstract{%
Dimensionality Reduction (DR) methods are widely used to visualize high-dimensional data. One key task in DR-based analysis is discovering neighborhoods, which relies on analyzing the fine-grained local structure of a projection. However, DR is an inherently lossy process; no technique can perfectly preserve the high-dimensional relationships, and projections therefore contain visual artifacts. In this paper, we highlight a typically overlooked source of visual artifacts: \textit{ambiguous instances}. These are instances that are highly similar to multiple mutually dissimilar neighborhoods in the high-dimensional space.
Standard DR methods cannot faithfully project such instances, since each data instance is mapped to a single point in the visual space. As a result, such an instance is placed in only one of its neighborhoods (or in none at all), so only part of its neighborhood structure is represented. We call this distortion \textit{partial neighborhood embedding}. In this paper, we introduce a graph-based approach that identifies ambiguous instances and replicates them as multiple points in the projection, placing each copy within its respective neighborhood. We use UMAP for our results, but our approach also generalizes to other local graph-based DR techniques. We show that our approach reveals previously hidden neighborhood memberships in projections and reduces partial neighborhood embedding across multiple examples, and is further supported by quantitative analyses. 
}

\keywords{Dimensionality Reduction, Ambiguity, Partial Neighborhood Embedding}

\teaser{
  \centering
  \includegraphics[width=\linewidth, alt={A view of a city with buildings peeking out of the clouds.}]{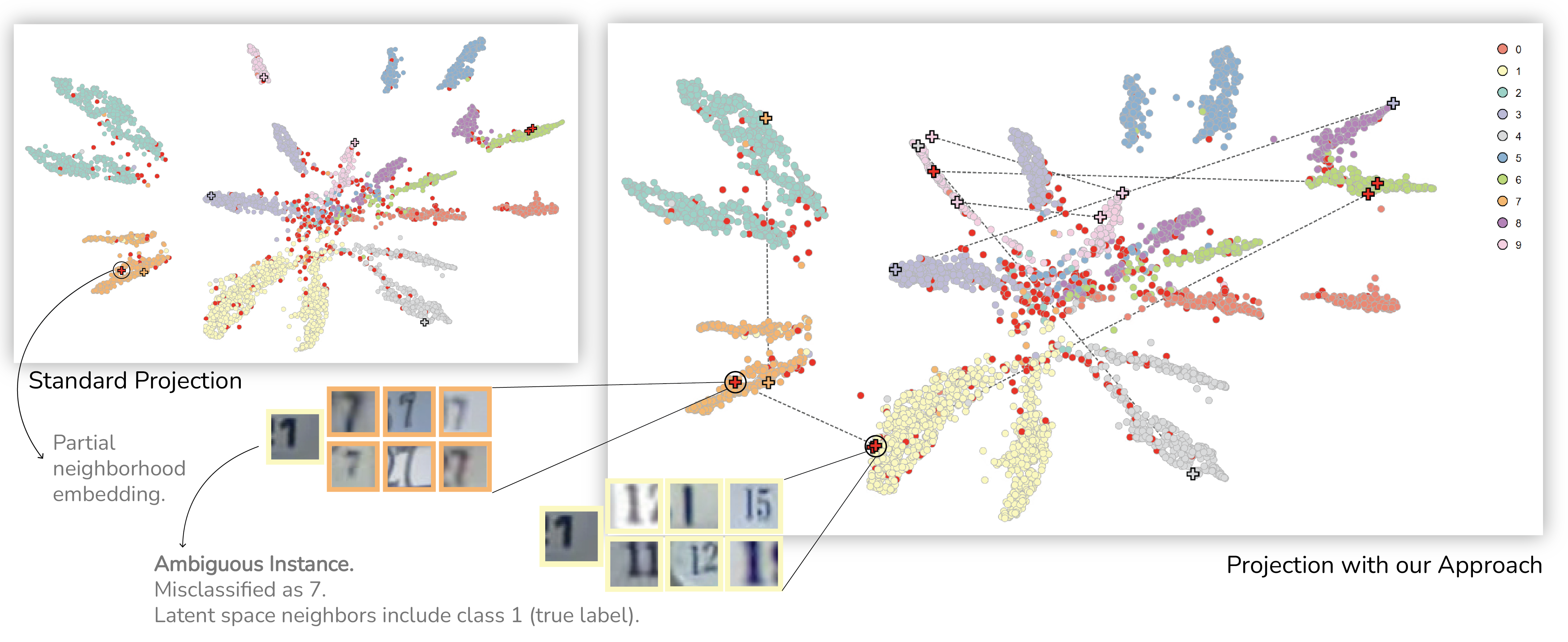}
  \caption{%
   We project the latent space of a CNN classifier trained on the Street View House Numbers dataset to study misclassifications (denoted in red). In a standard UMAP projection (left), several \textit{ambiguous instances} (denoted by crosses) suffer from \textit{partial neighborhood embedding}. Using our approach (right), these instances are split, revealing membership of multiple neighborhoods. For example, the highlighted instance with true label ‘1’ is misclassified as ‘7’, and appears only among sevens in the standard projection, but our approach also shows its similarity to its true class. \label{fig:svhn}}%
  }

\graphicspath{{figs/}{figures/}{pictures/}{images/}{./}} 

\usepackage{tabu}                      
\usepackage{booktabs}                  
\usepackage{lipsum}                    
\usepackage{mwe}                       
\usepackage{ccicons}                   

\usepackage{mathptmx}                  
\usepackage{amsmath}
\usepackage{amssymb}
\usepackage{amsthm}
\usepackage{mathrsfs}
\usepackage{todonotes}

\newtheorem{definition}{Definition}[section]
\newtheorem{theorem}{Theorem}[section]
\newtheorem{lemma}[theorem]{Lemma}

\usepackage{xcolor}
\colorlet{newtextcolor}{black}

\begin{document}
\crefname{definition}{Def.}{Defs.}
\Crefname{definition}{Definition}{Definitions}
\crefname{theorem}{Thm.}{Thms.}
\Crefname{theorem}{Theorem}{Theorems}
\crefname{lemma}{Lem.}{Lems.}
\Crefname{lemma}{Lemma}{Lemmas}



\maketitle

\section{Introduction}
Dimensionality Reduction (DR) is frequently used to visualize and interpret high-dimensional datasets. DR techniques project data from a high-dimensional space into a lower-dimensional space (typically 2D) while aiming to preserve as much of the original structure and relationships as possible~\cite{espadoto2019toward}. DR is popular in visualization and visual analytics, but its use has also expanded to other domains, such as cell or molecular biology~\cite{10.1371/journal.pcbi.1012403, greenCellularCommunitiesReveal2024}, proteomics~\cite{PhysRevX.11.041026, Viegas2024}, and drug discovery~\cite{wanDeeplearningenabledAntibioticDiscovery2024}. Among the analytical tasks supported by DR, discovering neighborhoods, which involves analyzing and interpreting fine-grained local structures, is among the most essential~\cite{nonato2018multidimensional}. Practical examples involve single-cell analysis, where the precise positioning of a single cell can be relevant, for example, to infer rare subpopulations~\cite{kobakArtUsingTSNE2019a, chariSpeciousArtSinglecell2023} or for local interpretability in Explainable AI (XAI), where neighborhoods in the latent space of a model can be examined to audit the classification of a specific data instance~\cite{rauberVisualizingHiddenActivity2017, ribeiro2019visual}.

Despite its relative success and widespread use, DR is an inherently lossy process; no DR method can perfectly preserve the manifold structure of high‑dimensional data. As a result, embeddings are approximations and inevitably contain visual artifacts, that is, visual information that does not faithfully represent the high-dimensional data. An example is manifold tearing, where points that are close in the high-dimensional space are placed too far apart in the embedding, distorting the topology of the original neighborhood~\cite{aupetitVisualizingDistortionsRecovering2007}. The literature~\cite{nonato2018multidimensional, martinsVisualAnalysisDimensionality2014} has cataloged other distortions caused by DR approximations. Here, we introduce an often-overlooked source of visual artifacts: \textit{ambiguous instances}. These are \textit{instances that are highly similar to multiple mutually dissimilar neighborhoods in the high-dimensional space.}

To clarify this concept, we draw an analogy to lexical ambiguity, where a word can belong to multiple dissimilar contexts. For example, the word \textit{jaguar} may refer to a large feline or a luxury car manufacturer. Standard DR methods cannot faithfully represent this ambiguity because each data instance is mapped to a single point in the visual space. In practice, this means that each point representing an ambiguous instance is likely to be placed in only one of its neighborhoods (or contexts) or in no neighborhood at all. We refer to this artifact as \textit{partial neighborhood embedding}, characterized by the loss of neighborhood information for ambiguous instances. The artifact is especially problematic in analyses that rely on instance-level neighborhood structure, as it can convey incomplete or incorrect information.

Metrics that measure neighborhood preservation, such as trustworthiness and continuity~\cite{kaski2003trustworthiness}, cannot properly capture this type of artifact. Although they may indicate reduced neighborhood preservation for these points, they cannot distinguish partial neighborhood embedding from other distortions, since such metrics operate on top of the resulting embeddings; they cannot distinguish whether problems arise from the embedding optimization or from ambiguous instances. Prior work such as multiple maps t-SNE~\cite{van2012visualizing} proposes using multiple projections to visualize non-metric similarities, but these become difficult to compare and this approach does not explicitly identify ambiguous instances. GhostUMAP2~\cite{jung2025ghostumap2} projects points to multiple positions in a single projection, but focuses on instability due to stochastic variation of the placement process rather than ambiguity arising from the data itself.

In this paper, we propose a technique specifically designed to tackle and mitigate the partial neighborhood embedding artifacts. First, we identify ambiguous instances in the high-dimensional space under some definition of neighborhood, for instance, based on the neighborhoods defined by a local DR technique, such as UMAP~\cite{mcinnes2018umap} or t-SNE~\cite{van2008visualizing}. Once detected, we use this information to resolve each ambiguous instance by ``splitting" its projection into multiple points, one for each of its distinct neighborhoods. \Cref{fig:jaguar} illustrates this idea; the word \textit{jaguar} is identified as ambiguous, and a standard projection would result in a partial neighborhood embedding. To address this ambiguity, we split the instance into two points, one corresponding to each neighborhood (or context), indicating that such an instance can belong to distinct neighborhoods. In summary, our main contributions are:
\begin{itemize}[itemsep=0mm, parsep=0pt]
  \item We introduce a DR artifact, which we call \textit{partial neighborhood embedding} that arises from \textit{ambiguous instances}, an often-overlooked source of distortion.
  \item We present a graph-theoretic approach to identify ambiguous instances using neighborhood relationships captured by a local DR technique.
  \item We propose an approach to split and project ambiguous instances, representing them with multiple points, achieving near-linear time complexity in practice.
\end{itemize}
The remainder of this paper presents related work (\cref{sec:relw}), proposes a definition for ambiguity in local DR and delineates our methodology to detect and duplicate ambiguous instances (\cref{sec:method}), showcases examples to illustrate our approach (\cref{sec:results}), provides some quantitative analysis (\cref{sec:quant}), and concludes with discussions and future directions (\cref{sec:limitations,sec:conclusions}). Our implementation is publicly available~\cite{code}.

\begin{figure}[]
    \centering
    \includegraphics[trim={0 0 0 0.8cm}, clip, width=\linewidth]{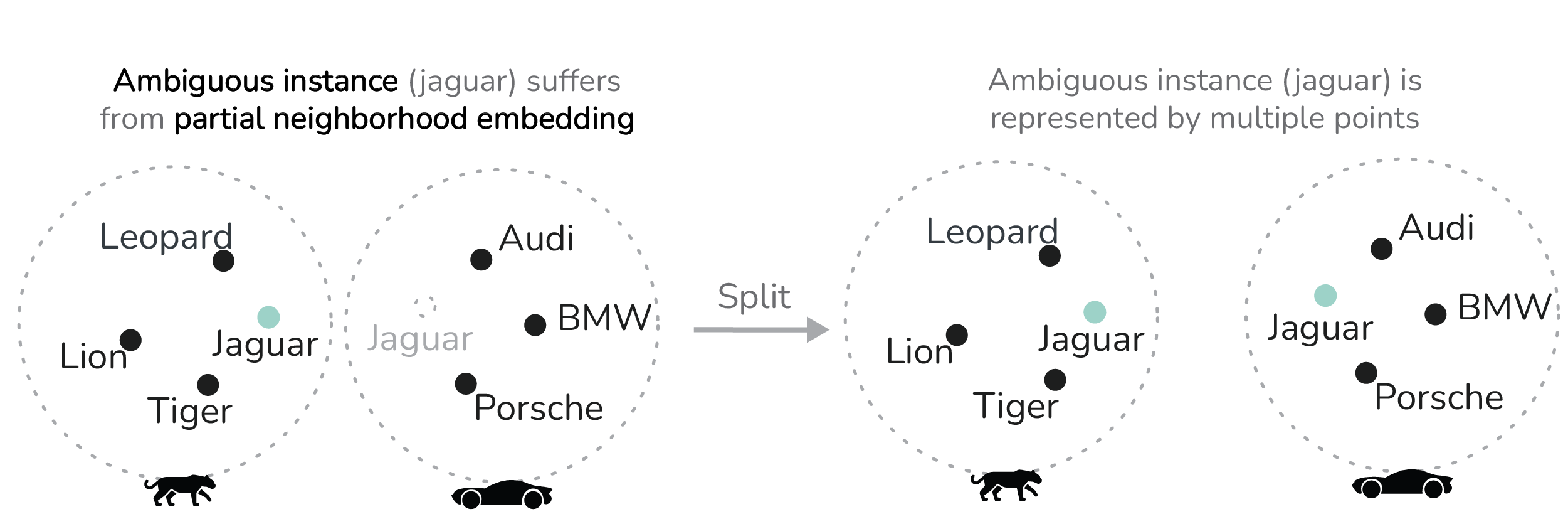}
    \caption{An example of the ambiguous data instance \textit{jaguar} that is placed near only one of its neighborhoods (partial neighborhood embedding). Our approach addresses this artifact by giving a copy of the instance to each distinct neighborhood (or context).}
    \label{fig:jaguar}
\end{figure}

\section{Related Work}\label{sec:relw}
In the following, we frame our work in DR and graph theory literature. Formally, DR maps each instance $x_i \in X$, where $X = \{x_1, \ldots,x_N\} \in \mathbb{R}^m$ into a point $y_i$ in the lower dimensional visual space, resulting in $Y = \{y_1,\ldots,y_N\} \in \mathbb{R}^{n=\{2,3\}}$~\cite{nonato2018multidimensional,espadoto2019toward}. Numerous DR techniques have been proposed. A common way to classify such techniques is based on the nature of the relationships they aim to preserve~\cite{nonato2018multidimensional,espadoto2019toward}. Global techniques, such as Multidimensional Scaling (MDS)~\cite{torgerson1952multidimensional}, ForceScheme~\cite{ROS2026104536, tejada2003improved}, or Sammons's Mapping~\cite{sammon1969nonlinear}, have the goal of preserving overall pairwise distances, whereas local techniques, such as t-Distributed Stochastic Neighborhood Embedding (t-SNE)~\cite{van2008visualizing}, Uniform Manifold Approximation and Projection (UMAP)~\cite{mcinnes2018umap}, ISOMAP~\cite{tenenbaum2000global}, Least-Square Projection (LSP)~\cite{paulovich2008least}, or Locally Linear Embeddings (LLE)~\cite{roweis2000nonlinear} aim to preserve local neighborhoods. 

\vspace{0.20cm}
\noindent \textbf{Topology in DR. } In our discussion, we refer to a recent framework by Paulovich et al.~\cite{paulovich2024dimensionality} which draws parallels between graph theory and DR. It formalizes DR as a two-phase process: first, the \textit{relationship phase} (or topology extraction) models the relationships between instances in the high-dimensional space, and then the \textit{mapping phase} transforms these relationships to a low-dimensional embedding. We will use this terminology in the remainder of the paper.

There are several DR techniques in which the relationships between instances can be modeled by constructing a topological structure, typically a variant of a Nearest Neighbor Graph (NNG). For example, UMAP uses simplicial complexes to construct a graph of the high-dimensional data. LargeVis~\cite{tang2016visualizing} also constructs a k-NNG. TopoMap~\cite{doraiswamy2020topomap} and its successor TopoMap++~\cite{guardieiro2024topomap++} focus on preserving the connected components (i.e., the 0-dimensional topology) through the use of the Rips filtration.  Each of these methods then maps this structure to the visual space through a layout algorithm; UMAP, both TopoMap variants, and LargeVis use a force-directed layout. There are also DR approaches that incorporate topological information explicitly during the mapping phase~\cite{heiter2024incorporating}. Our approach also captures high-dimensional relationships through topology and maps them using a graph-drawing technique. However, we use the topology to identify ambiguous instances in our high-dimensional data and adapt our graph accordingly before the mapping phase.

\vspace{0.20cm}
\noindent \textbf{Splitting and Ambiguity in DR. }
Ambiguity in high-dimensional spaces has largely remained unexplored. However, some existing approaches also consider using multiple representations for the same projection or instances to convey additional information.
One such topic is the study of instability in the mapping phase. Projection Ensemble~\cite{jung2023projection} creates graphs from multiple projections of the same dataset and then uses frequent subgraph mining to assess the stability of subgraph structures across different projections. GhostUMAP2~\cite{jung2025ghostumap2} focuses on conveying instability that arises from the stochasticity in the mapping process based on a single projection, and thereby also displays multiple points for the same instance. Barahimi~\cite{barahimi2021multi} allows an instance to appear at most twice, separating reliable and less reliable representations into two layers based on the forces applied during the mapping phase. Thus, these approaches focus on instability in the projections that arises from the mapping phase. While ambiguity may manifest as instability in the projection, we aim to find ambiguous instances based on the topological structure of the relationships. Taking a different perspective, van der Maaten and Hinton proposed multiple maps t-SNE~\cite{van2012visualizing}, which aims to faithfully represent non-metric similarities. Their approach does not rely on topological information, but rather on probabilistic neighborhood preservation through pairwise similarity distributions. A disadvantage of this approach is that it yields numerous projections of the same dataset, making visual comparison challenging. While it may be possible to deduce the existence of (some) ambiguous instances in the data through this approach, this is not its objective. Additionally, our approach yields a single projection rather than multiple.

\vspace{0.20cm}
\noindent \textbf{Ambiguity in Graph Theory and Network Visualization. }
Livi and Rizzi introduce the concept of graph ambiguity~\cite{liviGraphAmbiguity2013}. In their formulation, graphs are modeled using fuzzy‑set and fuzzy‑information theory, and ambiguity is defined as a global property of the graph, quantified through fuzzy entropy. This measure captures the number of distinct ways the graph can be interpreted, and is used as a discriminative descriptor for graph‑level tasks such as graph classification and inexact graph matching. In contrast, our notion of ambiguity concerns a fundamentally different phenomenon; it is not based on fuzzy‑set theory and is not a global graph property. Instead, we consider individual instances that belong to multiple disjoint neighborhoods in DR‑constructed graphs, and treat such instances as ambiguous.
In network visualization, ambiguity typically refers to layout-induced visual artifacts, including clutter, overlapping communities, or inconsistencies that arise from the drawing algorithm itself that impede the understanding of network structure~\cite{yong2016ambiguityvis, yan2017visualizing}. In contrast, the ambiguity we address originates in the data: an instance belongs to multiple mutually dissimilar neighborhoods in the high-dimensional space. This visual artifact appears only because standard DR methods must map each instance to a single position. Thus, although both manifest as visual artifacts, layout‑induced ambiguity arises from the drawing process, whereas our ambiguity originates in the data itself.

\vspace{0.20cm}
\noindent \textbf{Vertex Splitting. }
While splitting instances into multiple points is uncommon in DR, the concept of \textit{vertex} splitting exists in graph theory. Vertex splitting removes a vertex from the graph, and replaces it with two (or more) new vertices, and distributes the edges of the old vertex over the new vertices. An example of use is planarization, where the aim is to find the minimum number of vertex splits needed to turn a graph into a planar one~\cite{nollenburg2025planarizing, eppstein2018planar}.
Henry et al. discuss vertex duplication as a means to improve the readability of clustered social network layouts~\cite{henry2008improving}. Their work focuses on \textit{clustering ambiguity}, i.e., cases where an actor belongs to multiple communities and explores how duplicating such actors can support community- and graph-level readability tasks. Their approach is primarily conceptual, offering design guidelines without a formal method to identify or split such vertices. In contrast, we adopt a topological perspective and introduce an concrete approach to address this problem.

\section{Methodology}\label{sec:method}
\begin{figure*}[htb]
    \centering
    \includegraphics[trim={0 0 0 0.10cm}, clip, width=\linewidth]{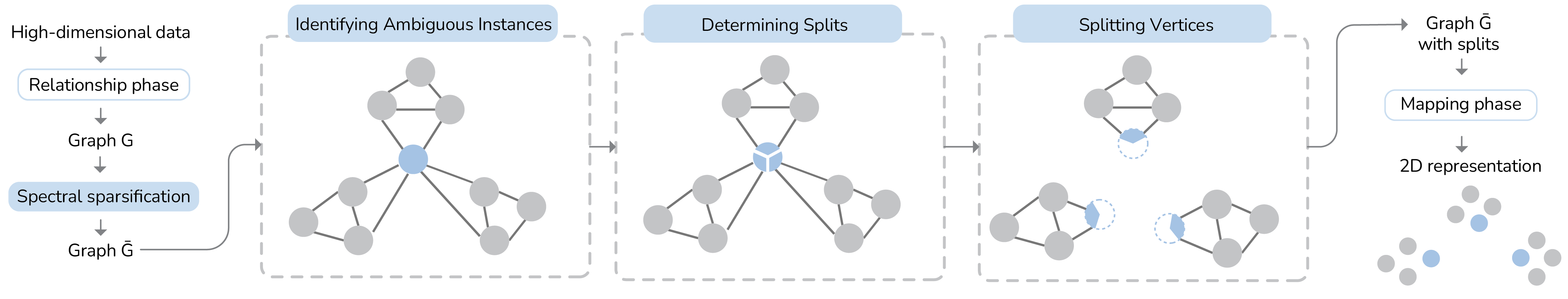}
    \caption{Schematic overview of our approach. Given a graph $G$ resulting from the relationship phase of a local DR technique, we sparsify $G$ to obtain $\bar{G}$, then use it to identify ambiguous instances, i.e., those that connect dissimilar neighborhoods. We determine how many ``copies'' or splits a vertex should get, and split the graph accordingly. This results in a graph that can be drawn in the mapping phase to create a disambiguated layout.}
    \label{fig:flowchart}
\end{figure*}

\begin{figure}[bp]
    \centering
    \includegraphics[width=1.0\linewidth]{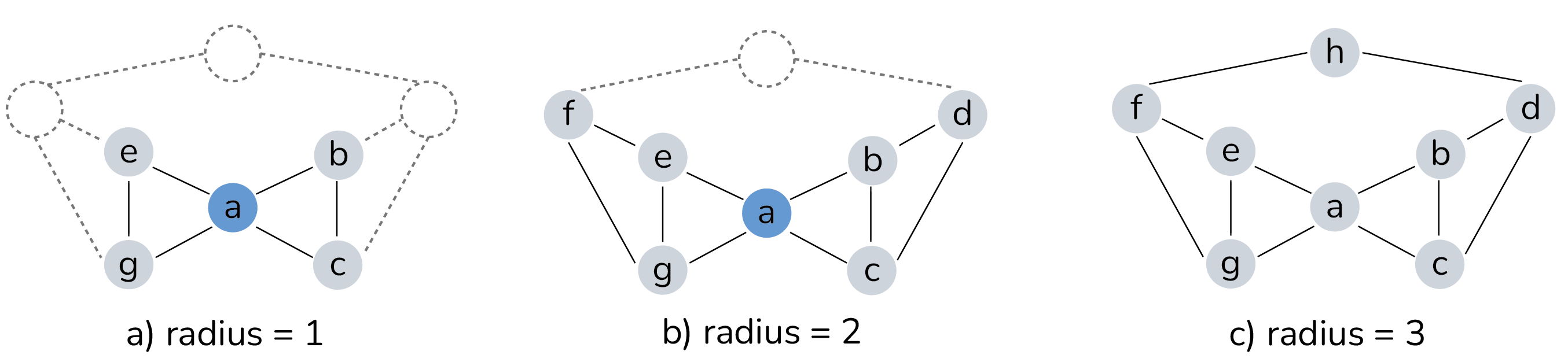}
    \caption{Vertex $a$ is identified as $LAP$ at $r = 1$ and at $r = 2$ because the neighborhoods are disjoint if $a$ was removed. At $r = 3$, vertex $h$ connects the previously disjoint neighborhoods.}
    \label{fig:LAP}
\end{figure}

\vspace{0.20cm}
\noindent\textbf{Problem Definition and Approach Overview. } 
As previously mentioned, the effect of data ambiguity in DR has received limited attention. We informally define  this phenomenon as follows: 

\begin{quote}
\textit{Ambiguous (data) instances} are instances that are highly similar to multiple mutually dissimilar high-dimensional neighborhoods.
\end{quote}

When such data are projected into a low-dimensional space, the neighborhood structure cannot be properly preserved because neighborhoods will be spread across the visual space, and a point cannot occupy multiple positions simultaneously. As a result, ambiguous instances are typically positioned near only one of the neighborhoods to which they belong, or in no neighborhood at all, obscuring their relationship to multiple neighborhoods. This is particularly problematic when using DR for fine-grained local analysis. Thus, unresolved ambiguity typically manifests as visual artifacts in DR projections. This phenomenon has similar roots to \textit{hubness}~\cite{10.5555/1756006.1953015}, a side-effect of the curse of dimensionality in which some high-dimensional data instances disproportionately appear among the nearest neighbors of many instances~\cite{Feldbauer2019}. The fundamental difference between them is that \textit{hubness} concerns instances that are nearest neighbors of many others, whereas ambiguity concerns instances that are members of multiple dissimilar neighborhoods.  In contrast, \textit{antihubs}~\cite{Feldbauer2019} rarely or never occur among the nearest neighbors of other instances and are therefore unlikely to be ambiguous.

We adopt the framework introduced by Paulovich et al.~\cite{paulovich2024dimensionality} to formally define this ambiguity phenomenon for local DR techniques and to present a potential solution based on instance splitting. This framework decomposes the DR process into two phases. In the \textit{relationship phase}, the relationships aimed to be represented, such as neighborhood or distance, are modeled using an undirected weighted graph $G$. In the \textit{mapping phase}, the data is projected accordingly by embedding $G$ into the visual space. For the most popular local techniques, such as t-SNE and UMAP, the embedding is computed by balancing a system of attractive forces between adjacent vertices and repulsive forces between non-adjacent vertices in $G$. Thus, the edges in $G$, i.e., the modeled relationships, play a key role in determining how the data are arranged in the visual space. We therefore express ambiguity directly in terms of the structure of $G$.

Figure~\ref{fig:flowchart} summarizes our approach. After modeling the DR relationships as $G$, we use spectral sparsification to obtain $\bar{G}$, reducing redundant edges that obscure meaningful neighborhood structure (\cref{sec:step1}). We then identify ambiguous vertices in $\bar{G}$ (\cref{sec:step1}), adapting a strategy from graph theory. For each ambiguous vertex, we determine how many neighborhoods it connects, and split the vertex to have a ``copy'' per neighborhood (as detailed in \cref{sec:step2}). This yields a modified graph that can be embedded in the visual space (\cref{sec:embedding}), producing a layout in which multiple representations of an instance reflect its multiple neighborhood memberships.

To set notation, let $G = (V, E)$ denote the undirected weighted relationship graph, where $V = \{v_1, \ldots, v_N\}$ is the set of data instances, $E = \{e_\textit{ij} := (v_i, v_j, \omega_\textit{ij})\} \subseteq V \times V$, where $\omega_\textit{ij}$ denotes the edge weight of $e_\textit{ij}$ between vertex $v_i$ and $v_j$. Next, we discuss each step of our approach.

\subsection{Identifying Ambiguous Vertices}\label{sec:step1}
In Graph Theory, the notions of \textit{articulation points}~\cite{tarjan1972depth} or \textit{cut vertices}~\cite{harary2018graph} approximate our notion of ambiguity, defining nodes whose removal splits a graph into multiple separate components, effectively connecting otherwise disconnected neighborhoods. However, in our setting, ambiguity does not necessarily manifest as a true articulation point, since neighborhoods may still be linked by weak or incidental edges that have minimal impact on the attraction/repulsion forces, and therefore, on the resulting layout. For instance, local neighborhood graphs tend to become denser as neighborhood size increases, introducing many redundant or marginal edges arising from incidental proximity rather than genuine neighborhood membership. 
These spurious connections might obscure articulation points, making them more difficult to detect: hence, we first sparsify the graph to better expose the neighborhood structure.

\vspace{0.20cm}
\noindent\textbf{Graph Sparsification. } To reveal the underlying neighborhood structure, we apply spectral sparsification. This results in a reduced graph whose Laplacian quadratic form remains a close approximation of the original graph~\cite{spielmanGraphSparsificationEffective2008}. This process removes edges that have little influence on the Laplacian, yielding a sparser graph that preserves its structural, connectivity, and algebraic properties, thus yielding a similar embedding of the original graph~\cite{IMRE202089}. In contrast, other sparsification strategies, such as simple weight-threshold filtering, do not guarantee preservation of the graph’s global structure. Although the spectral guarantees are global, the goal is to remove redundant or incidental edges that have little influence on the final DR layout.

Specifically, we use the Spielman-Srivastava sparsification algorithm~\cite{spielmanGraphSparsificationEffective2008}, as this has a nearly linear time complexity. This approach is based on effective resistance, which measures how much an edge contributes to the overall connection structure of the graph. The idea is that if the graph represents an electrical network, each edge $e_\textit{ij}$ with weight $\omega_{\textit{ij}}$ is viewed as a resistor with a resistance of $1/\omega_{\textit{ij}}$. The effective resistance of an edge $e_\textit{ij}$ is defined as the voltage difference that appears between its endpoints when a unit current is injected at one endpoint and extracted at the other. Edges are randomly sampled with probability proportional to their effective resistance, so more important edges are more likely to be retained, whereas low-resistance (redundant) edges are more likely to be pruned. Each retained edge is reweighted inversely proportional to its sampling probability to preserve the overall connectivity.

Let $L$ denote the Laplacian of the original graph $G$, that is, the matrix $L = D - A$, where $D$ is the diagonal degree matrix with $D_{ii} = \sum_j \omega_\textit{ij}$ and $A$ is the adjacency matrix. Also, let $\bar{L}$ denote the Laplacian of the sparsified graph $\bar{G}$. The Spielman-Srivastava sparsification algorithm guarantees that: 

\begin{equation}
(1 - \epsilon)x^TLx \leq x^T\bar{L}x \leq (1 + \epsilon)x^TLx,
\end{equation}
holds for all vectors $x$, where $\epsilon \in (0,1)$ is an error margin. The parameter $\epsilon$ controls a trade-off between precision and sparsity; a lower $\epsilon$ retains more edges and finer details, while a higher $\epsilon$ produces a sparser graph. Thus, varying $\epsilon$ allows us to influence what qualifies as a ``weak'' connection in the context of ambiguity. For more details, refer to Spielman and Srivastava~\cite{spielmanGraphSparsificationEffective2008}.

\vspace{0.20cm}
\noindent\textbf{Local Articulation Points.} While sparsification exposes meaningful structure required for articulation analysis, articulation points (as described by Tarjan~\cite{tarjan1972depth}) only reflect connectivity on a global level. However, we consider ambiguity to be a local phenomenon that arises between neighborhoods rather than at the scale of the entire graph. Therefore, to better align with the objective of local DR methods, we adopt \textit{local articulation points (LAPs)}~\cite{yu2018LAP}. These measure articulation behavior within an $r$-step neighborhood, which, for a given vertex $v_i \in V$, is defined as

\begin{equation}
G_r(v_i) = \{\, v_j \in V \mid d(v_i, v_j) \leq r \,\},
\end{equation}
where $d(v_i, v_j)$ denotes the unweighted shortest-path distance (hop count) between $v_i$ and $v_j$ in $G$, and $G_r(v_i)$ is an induced subgraph of $G$.

The concept of local articulation points (LAPs) was introduced by Yu et al.~\cite{yu2018LAP}, primarily for analyzing network attacks. We provide a precise mathematical definition to formalize this notion by defining the set of $r$-step local articulation points on a graph $G$, denoted $LAP_r$. This set consists of all vertices whose removal increases the number of connected components within their $r$-step induced neighborhoods, where a \textit{connected component} is a maximal set of vertices in $G$ which all vertices can reach one another. Here, $r$ is a parameter of the approach, which determines the locality considered.
\begin{definition}[$r$-step Local Articulation Point] 
A vertex $v_i$ is an $r$-step Local Articulation Point (LAP) if its removal increases the number of connected components of the subgraph of $G$ induced by $G_r(v_i)$.
\end{definition}
The maximum radius $r_{max}$ for any vertex $v_i$ is defined by the distance from $v_i$ to its farthest reachable vertex.
Note that no $r_{max}$ can exceed the graph diameter, that is, the maximum shortest-path distance between any two vertices in $G$. Let $\mathcal{C}(\cdot)$ to denote the set of connected components of a graph. The set of all LAPs at radius $r$ can be written as
\begin{equation}
    LAP_r(G) = \{\,v_i \in V \mid |\mathcal{C}(G_r(v_i) \setminus v_i)| > |\mathcal{C}(G_r(v_i))|\,\}.
\end{equation}

Given this definition and the preceding discussion, we formally define ambiguity as follows.
\begin{definition}[Ambiguous Instance at $r$]\label{def:ambig}

Let $\bar{G} = (V,\bar{E} \subseteq E)$ be the sparsified version of the graph $G$, and let $r$ be a radius.
An instance $v \in V$ is considered ambiguous at radius $r$ if $v \in LAP_r(\bar{G})$.
\end{definition}

To illustrate the concept, we consider a graph in \cref{fig:LAP} for which we aim to find $LAP_r$. For each vertex, we extract its induced subgraph at $r$. \Cref{fig:LAP} shows three such induced subgraphs for vertex $a$, each at a different radius. While $a \in LAP_1$ and $a \in LAP_2$, $a \notin LAP_3$. At $r = 1$ removing $a$ creates two separate connected components; $\{b, c\}$ and $\{g, e\}$. At $r = 2$, the induced subgraph contains neighboring vertices and edges of points $b, c, g \text{ and } e$. Removing $a$ creates two separate connected components; $\{b,c,d\}$ and $\{e,f, g\}$. At $r=3$, vertex $h$ connects the two previously separate components, thus $a \notin LAP_3$. Note, however, that at $r=1$, only direct neighbors are considered, yielding limited structural information.

Computing $LAP_r$ at every radius can be computationally expensive, as it requires checking the connectivity of the induced subgraph for every vertex in the relationship graph. However, we exploit Lemma~\ref{lemma:lemma} to reduce the number of iterations required, where $D$ denotes the diameter of $\bar{G}$.
\begin{lemma}\label{lemma:lemma}
For any vertex $v_i \in V$ and any radius $1 < r \le D$ if $v_i \notin LAP_r$, then $v_i \notin LAP_{r'}$ for all $r' > r$.
\end{lemma}
\begin{proof}
If the induced subgraph $\bar{G}_r(v_i)$ at radius $r$ remains connected after removing $v_i$, then for any larger radius $r' > r$, the subgraph $\bar{G}_{r'}(v_i)$ with $v_i$ removed is also connected. This follows from the fact that, by construction, $\bar{G}_r(v_i) \subseteq \bar{G}_{r'}(v_i)$, thus $E_{\bar{G}_{r}} \subseteq E_{\bar{G}_{r'}}$. Therefore, if $\bar{G}_r(v_i) \setminus v_i$ is connected, $\bar{G}_{r'}(v_i) \setminus v_i$ is also connected.
\end{proof}
In practice, many vertices are never LAPs or are LAPs only at small radii, so this significantly improves the algorithm's efficiency. 

\subsection{Determining Splits and Splitting Vertices} \label{sec:step2}
\noindent \textbf{Determining Splits.} 
In the previous subsection, we discussed how to identify ambiguous data instances using LAPs; now we introduce the definition of a \textit{vertex split} and explain how we determine the number of splits per ambiguous vertex.
For every vertex $v_i \in LAP_r$ and radius $r$, to determine the number of splits, we consider the connected components that remain after removing $v_i$ from its induced subgraph $G_r(v_i)$. The splits are determined based on the graph $\bar{G}$, ensuring that the outcome for each vertex does not depend on the order in which vertices are processed.

Formally, we define a vertex split for any vertex $v_i \in LAP_r$ as follows.

\begin{definition}[Vertex Split]
    For a vertex $v_i \in LAP_r$ and radius $r$, a \textit{vertex split} replaces $v_i$ with a set of vertices ${S}(v_i,r) = v_{i,1},\dots,v_{i,K}$ where $K = |\mathcal{C}(G_r(v_i) \setminus v_i)|$.
    Each $v_{i,k} \in S(v_i,r)$ corresponds to a connected component $C_k \in \mathcal{C}(G_r(v_i) \setminus v_i)$ and inherits edges to the neighbors in that component:
    \[
    E(v_{i,k}) = \{ (v_{i,k}, v_j, \omega_\textit{ij}) \mid v_j \in \Gamma(v_i) \cap C_k \},
    \] where $\Gamma(v_i)$ are the neighbors of $v_i$ and $\omega_\textit{ij}$ is the edge weight between $v_i$ and $v_j$ in $\bar{G}$. 
\label{def:vsplit}
\end{definition}

\begin{figure}[tbp]
    \centering
    \includegraphics[width=1.0\linewidth]{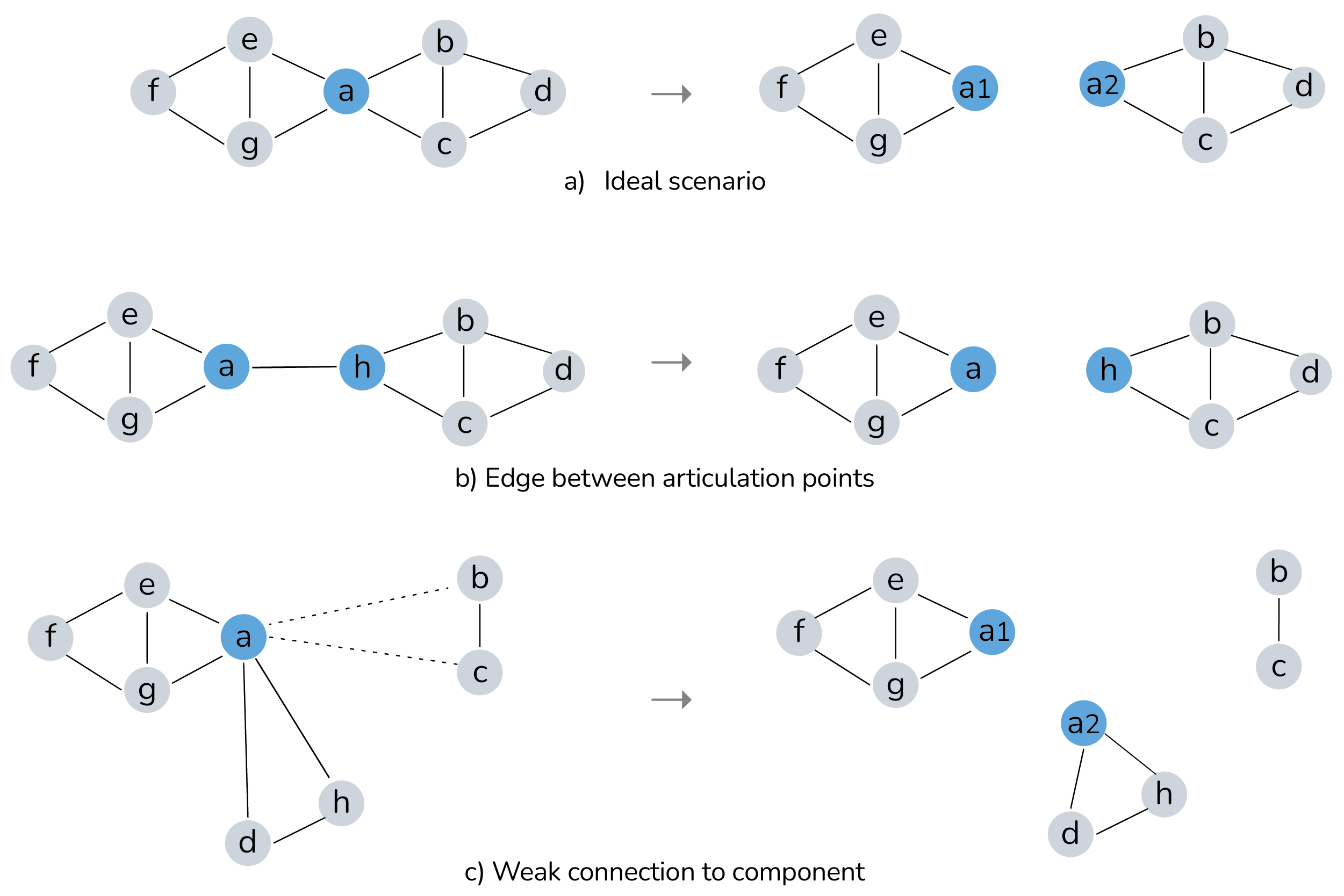}
    \caption{Three scenarios for vertex splitting, showing the resulting graph after a potential split. Ambiguous instances are shown in blue.}
    \label{fig:cases}
\end{figure}

\vspace{0.20cm}
\noindent\textbf{Decision Rules for Vertex Splits. }The most straightforward scenario for disambiguating an ambiguous vertex by splitting it into two is presented in \cref{fig:cases}a. 
In the ideal scenario, vertex $a$ is adjacent to two disjoint neighborhoods, so it is split into two vertices, $a1$ and $a2$, each adjacent to its respective neighborhood.

There are cases where vertices are identified as ambiguous, but given the graph's structure or weight, there should be no (or fewer) splits, as these splits may result in isolated points in the projection due to weak attraction forces. More specifically, there are two scenarios in which a split should not be included in $S(v_i,r)$, as illustrated in \cref{fig:cases}b and c. The former is an LAP--LAP connection: if an edge bridges two otherwise disjoint neighborhoods, and both endpoints are identified as LAPs (\cref{fig:cases}b). In this situation, the vertices are not genuinely ambiguous, but their multi-neighborhood structure is created only by a single edge. Therefore, no additional vertices are added to $S(v_i,r)$. Since this edge represents a spurious bridge between otherwise disjoint components, we do not include it in the resulting graph, while keeping both $a$ and $h$ unsplit. Another scenario is a weak connection to the component. If the aggregate weight of some connection to a component is significantly smaller than the weight of its strongest component, e.g. the component with the largest total edge weight to the LAP, we classify it as a ``weakly supported neighborhood''. In \cref{fig:cases}c, the connection of vertex $a$ to connected component $\{b,c\}$ is weak, while the connection to component $\{d,h\}$ and $\{e, f, g\}$ is strong. The split set for vertex $a$ consists of two splits instead of the potential three. We also exclude splits for components that consist of only one element. We do not include such splits from $S(v_i,r)$ to ensure that the disambiguated graph only contains splits that are supported by sufficient structural evidence.

\vspace{0.20cm}
\noindent\textbf{User-Defined Weight Threshold. } We introduce a user-defined threshold $\tau_w \in [0,1]$, which specifies how strong a component’s connection to a vertex must be, relative to the strongest connected component. A component's connection strength is defined as the sum of all weights $\omega_\textit{ij}$ from $v_i$ to all non-LAP vertices in that component. The threshold $\tau_w$ thus controls which components are included in $S(v_i,r)$; only those whose connection strength reaches at least a fraction $\tau_w$ of the strongest component are retained. Because $\tau_w$ is defined as a ratio rather than an absolute weight threshold, it adapts to local density. This is important as edge weights vary with local density, so a global cutoff may remove meaningful connections in sparser regions. 

\vspace{0.20cm}
\noindent \textbf{Constructing Disambiguated Graph. } \label{sec: splitting}
Finally, given a sparsified graph $\bar{G} = (\bar{V}, \bar{E})$, for every radius $r$ where $LAP_r \neq \emptyset$, we construct a disambiguated graph $\bar{G}'_r = (V'_r, E'_r)$ as follows. For every $v \in LAP_r$, we add its split set $S(v, r)$ to $V'_r$. All other vertices $u \in V \setminus LAP_r$ are added to $V_r'$.
Edges in $\bar{E}$ whose endpoints both lie in $LAP_r$ are not included in $E'_r$, while the opposite applies to edges incident to vertices both \textit{not} in $LAP_r$. All other edges and their weights are included in $E'$ according to the vertex splits, following \cref{def:vsplit}.

\subsection{Mapping and Visualization} \label{sec:embedding}
As a final step, the mapping phase is executed to define the spatial positions of the vertices in $\bar{G}_r'$, yielding a projection in which the splits of ambiguous instances are assigned to each of their neighborhoods. The aim is to produce an embedding with fewer partial neighborhood artifacts. Merely augmenting the current layout with additional points is not enough because it may retain structures that arise from these artifacts. This can be achieved using any mapping method that handles a (weighted) graph as input, such as the negative sampling strategy employed by UMAP or a suitable graph drawing technique, such as ForceAtlas2~\cite{10.1371/journal.pone.0098679}. Executing the mapping phase on $G_r'$ yields a projection with splits. Since each instance can appear as multiple ``copies'' in the graph, one for each neighborhood, points can be positioned near all of their disjoint or weakly connected neighborhoods. As a simple visual metaphor to represent this information, we distinguish splits from other points using different markers (crosses vs. circles, respectively), and render the splits last so they do not overlap with other points. In addition, we added an interaction that highlights and connects, using dashed lines, all copies of a selected ambiguous instance, making it easier to identify the different neighborhoods the copies were assigned. 

\subsection{Computational Complexity}
The computational cost of our approach consists of four main components. We take as input the relationship graph $G=(V,E)$ with $n = |V|$ vertices and $m = |E|$ edges. Sparsifying $G$ using the Spielman-Srivastava algorithm runs in near-linear time $\tilde{O}(m)$. Since we are considering only local DR techniques, with instances connected to only a few instances (usually the nearest neighbors), the sparsification is then $\tilde{O}(n)$. Identifying ambiguous vertices in $\bar{G}$ requires computing connected components of the induced neighborhoods $G_r(v_i)$. In the theoretical worst case, this is $O(n+m)$ per vertex, giving a $O(n^2)$ overall bound. In practice, this cost is much lower for sparser graphs: by Lemma~\ref{lemma:lemma}, once a vertex is not an LAP at some radius, it cannot become one at larger radii, so only a small number of radii must be tested, and relatively few vertices are LAPs. Determining splits and constructing each per-radius graph $G'_r$ requires only linear passes over the relevant neighborhoods. Finally, the cost of the mapping phase depends on the chosen DR or layout method, but since $\bar{G}$ is typically significantly smaller than $G$, and $G_r'$ differs from $\bar{G}$ by only a small number of additional vertices and edges, this step adds negligible overhead compared to embedding the original graph. Thus, although a quadratic worst case exists, in practice the overall computational complexity is near-linear in the number of data instances for detecting and splitting ambiguous instances; the drawing complexity follows the employed algorithm.

\section{Usage Scenarios}\label{sec:results}
\begin{figure*}[htb]
    \centering
    \includegraphics[width=0.9\linewidth]{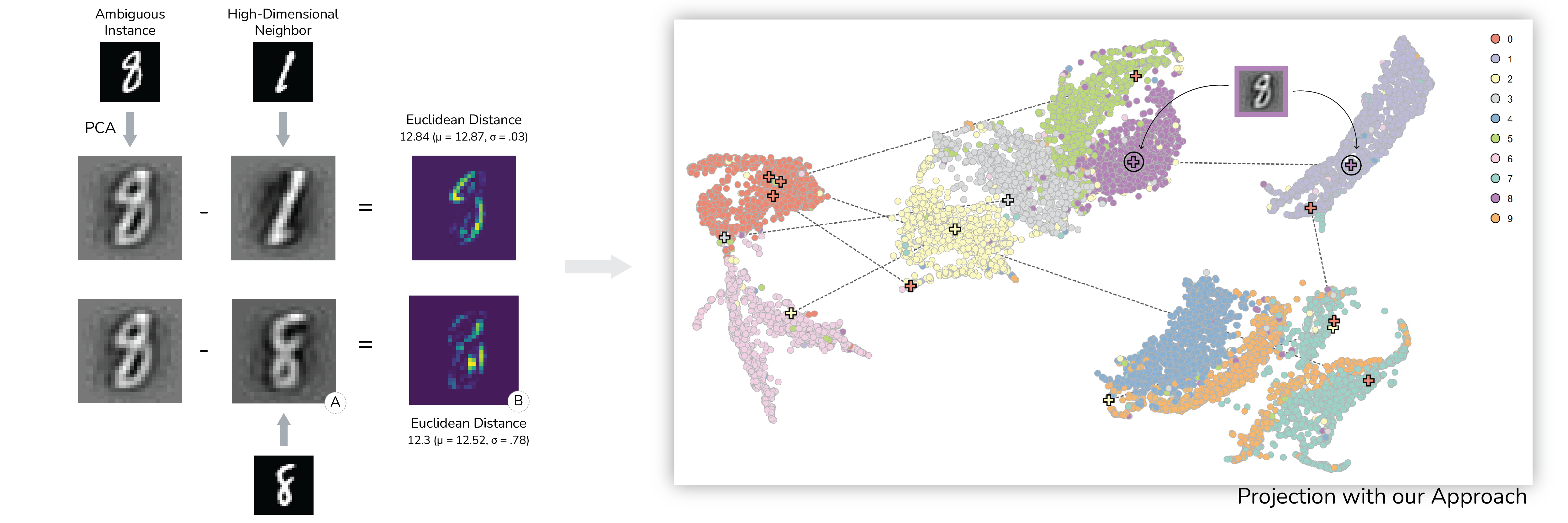}
    \caption{Projection of the MNIST testset, first reduced to $200$ dimensions through PCA. The highlighted ambiguous instance (8) may seem unintuitive, but based on the data representation and distance metric (Euclidean), it is similar to both 1s and 8s.}
    \label{fig:mnist}
\end{figure*}

\begin{figure*}[htbp]
    \centering
    \includegraphics[width=0.8\linewidth]{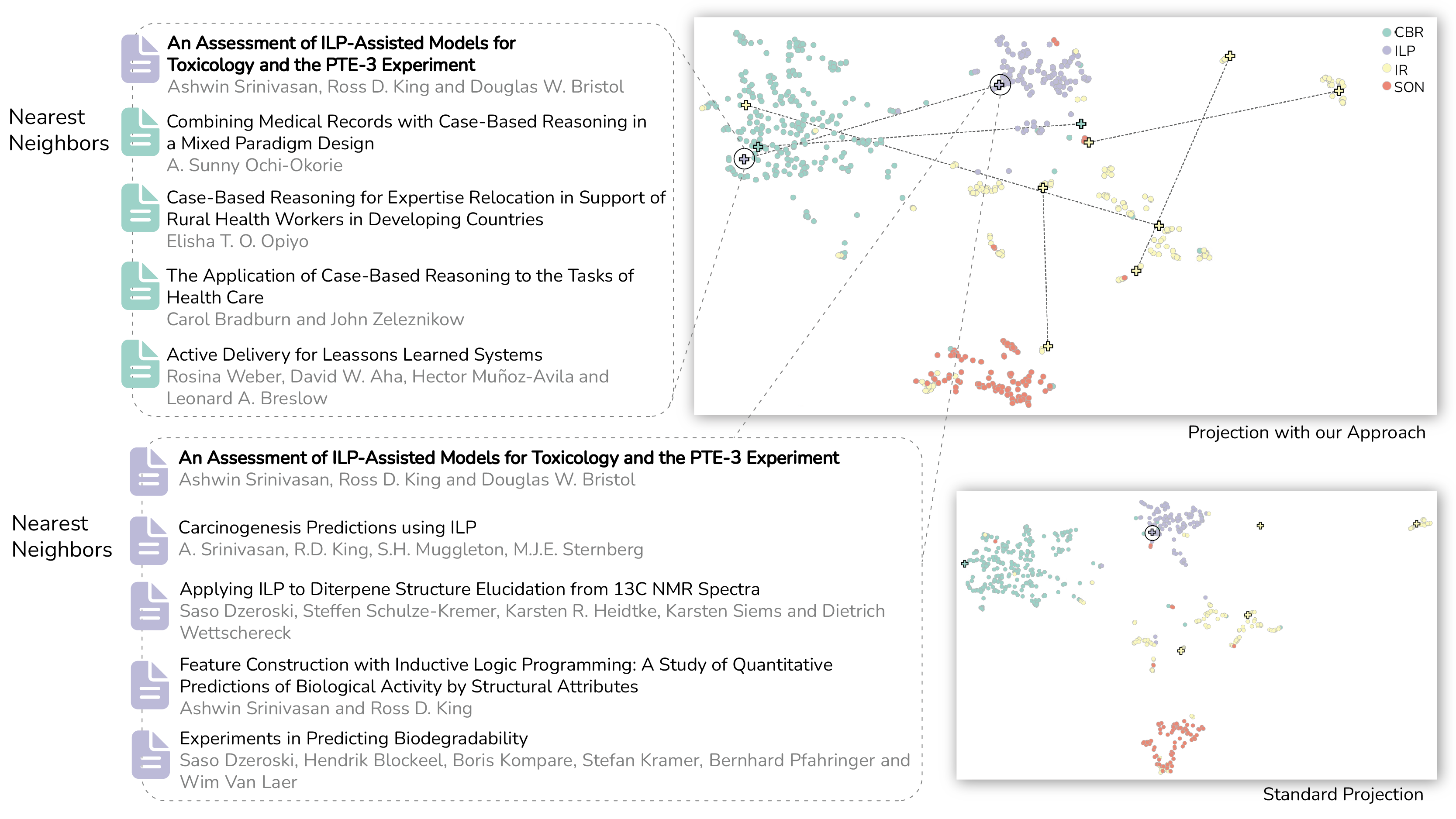}
    \caption{Two projections of the same text dataset containing information on papers, the standard projection suffers from partial neighborhood embedding for three instances (under this parameterization). The highlighted instance was placed with only papers on Inductive Logic Programming (ILP), but with our approach, we see that it is also similar to papers in Case-Based Reasoning (CBR).}
    \label{fig:cbr}
\end{figure*}

In this section, we present examples of how splitting ambiguous instances can reveal additional insights and improve the problem of partial neighborhood embedding. We also conduct analysis on the graph sparsification, the effect of neighborhood size, and whether existing DR quality metrics can capture ambiguity in the data. In the following examples, points representing ambiguous instances with splits ($|S(v_i,r)| > 1 $, i.e., vertices that are replaced by two or more vertices) are represented by a cross, while all other points are represented by circles. We connect multiple representations of the same instance through dashed lines. To facilitate comparison between a standard projection and a projection produced by our approach, we align the layouts as closely as possible by initializing our projection with the coordinates of the standard one and assigning positions to split points in this initialization based on the mean positions of their neighbors in the standard layout. The mapping phase is executed on the graph with splits using this initialization. Other alignment approaches exist (e.g., ~\cite{rauberVisualizingTimeDependentData2016, cantareiraGenericModelProjection2020}), but here we opt for simplicity. We use UMAP for both the relationship and mapping phases, although the approach is technique-agnostic for graph-based local DR.

\vspace{0.20cm}
\noindent\textbf{Enhanced Instance-Level Insights.}
In this example, we highlight how addressing the partial neighborhood embedding problem by splitting ambiguous instances can improve the accuracy of the derived insights in a local explainable Artificial Intelligence (XAI) example.
We train a convolutional network following the architecture from~\cite{rauberVisualizingHiddenActivity2017} on the cropped Street View House Numbers (SVHN) dataset~\cite{netzer2011reading}. We use the predefined train and test splits ($73,257$ and $26,032$ instances, respectively), and further divide the train set into a 90/10 training-validation split. The resulting model achieves an accuracy of $93\%$ on the test set. For our analysis, we extract the $512$-dimensional latent space of the model as inspired by~\cite{rauberVisualizingHiddenActivity2017} and select a random subset of 5,000 examples to examine some of the misclassifications in more detail. We set the number of neighbors of UMAP to $15$ since it is a commonly used value~\cite{mcinnes2018umap}. We set $\epsilon = 0.7$, since a sparser graph allows us to identify more ambiguous instances, and $\tau_w = 0.05$ so that we only limit splits that have $5\%$ of the weight to the strongest component per instance. We do this so that we can still capture many of the ambiguous instances; however, contributions below $5\%$ represent such a small fraction of an instance's overall weight that we want to disregard those. \Cref{fig:svhn} shows the standard projection (left) and the projection after executing our approach at $r=2$, which covers $2.17\%$ of the total graph on average, that is, on average $2.17\%$ of the graph is reachable if starting from any of its vertices at $r=2$. We use $r=2$ for this and all subsequent examples because it is the most local yet requires instances to be ambiguous on a larger scale.

For the local XAI analysis, we focus on a single instance, as is typical when auditing classification results. The selected (ambiguous) instance is misclassified as $7$ but has a ground-truth class of $1$. We can see in the projection that, after our approach identified the instance as ambiguous, it is split into two points, placing one duplicate within a neighborhood of sevens (the predicted label) and the other within a neighborhood of ones (the true label). Yet, when looking at the standard projection (\cref{fig:svhn}(left)), this instance is only shown in the neighborhood of sevens. This is a clear example of partial neighborhood embedding, where it was obscured that the instance is also similar to the neighborhood of ones. As a result, an analyst investigating misclassification causes~\cite{rauberVisualizingHiddenActivity2017} may be misled into believing that the classifier is confident about the mistake, even though this instance is actually ambiguous with respect to which neighborhood it belongs to. This ambiguity is valuable information when a single instance is of particular interest, yet it is typically hidden when standard DR approaches are applied.

\vspace{0.20cm}
\noindent\textbf{Expectation vs Data Representation.}
While the previous example may have been intuitive, as a one and a seven look alike, it is important to note that this effect results from the data representation (the model may have learned what numbers look like) and the employed distance metric. \Cref{fig:mnist} shows a UMAP projection of the MNIST test set~\cite{lecun1999object} (10,000 instances). A common way to reduce noise before projecting this data is to first apply PCA~\cite{van2008visualizing}. We apply PCA and reduce the data to $200$ dimensions to obtain a projection that better separates the different labels while retaining sufficient detail. Consistent with the previous example, we use our approach with $\epsilon = 0.7$, $\tau_w = 0.05$ and inspect the results at $r = 2$ (average coverage = $1.73\%$).

As an example, we inspect one of the ambiguous instances (\cref{fig:mnist}), labeled as an $8$, and its neighbors in the UMAP graph. Its neighbors are spread across two components; one consisting of $11$ instances with label $8$ (and one instance with label $2$), and consisting of three of instances with label $1$. Although we may assume that an eight does not look very similar to a one, in this example, we use Euclidean distance (squared pixel differences) over the PCA reduction as input to UMAP, which does not necessarily reflect our interpretation of how similar two instances are. To illustrate this, in \cref{fig:mnist} we inverse-project the PCA reductions to recreate the (approximate) images of the digits (A) as well as the differences between them (B). As can be noticed, the difference between two $8$s, or one $8$ and a $1$, pixel-wise, is close, causing the instances to appear as ambiguous even though we may not perceive it as such. Therefore, it is important to note that ambiguity is a property of the data representation and the distance metric used, rather than of how we perceive it. In an ideal scenario, expectation and representation should match, but this is not always the case, and our approach can reveal such mismatches, and could thus be used to explore whether the distance metric matches the mental model of the user.

\vspace{0.20cm}
\noindent\textbf{Application on Text Data.}  While the previous two examples concern image data, our approach also works for other types of data, such as textual data. We use a dataset consisting of the titles, authors, abstracts, and references for $675$ scientific papers, divided into four classes, originally introduced in~\cite{paulovich2008least}. The four classes are Case-Based Reasoning (CBR), Inductive Logic Programming (ILP), Information Retrieval (IR), and Sonification (SON). Papers are assigned to these classes based on which of these terms were used in the querying process. We use a Bag-of-Words (BoW) model, weighted by Term Frequency–Inverse Document Frequency (TF-IDF), to vectorize the data, and cosine similarity as the distance metric for the DR technique, the typical choice for this type of data~\cite{manning2008introduction}. We set UMAP's number of neighbors to $10$ as this is a smaller dataset. Again, we set $\tau_w$ to 0.05. We set $\epsilon = 0.7$. At $r=2$, the average coverage of the graph is $9.8\%$. 

Both the standard projection and the projection where ambiguous instances are split are shown in \cref{fig:cbr}. There are three ambiguous instances that exhibit partial neighborhood embedding issues in the standard projection. We detail one such instance and its nearest neighbors in the visual space in \cref{fig:cbr}. The ambiguous instance is a paper on \textit{ILP-Assisted Models for Toxicology}, and in this case is represented by two points in the layout; one is positioned along the other papers on ILP, and the other is in a group of papers on CBR. However, once we look at the titles of these papers (and inspect the abstracts), we can see that while the nearest neighbors in the ILP group are clearly related to ILP as well as healthcare and biology, the CBR neighbors are about medical records and healthcare and designing systems to support this, so these papers have related terms and subjects. Therefore, this instance can belong to both neighborhoods, but a standard UMAP projection would have obscured this by placing it in only one neighborhood. This shows a different scenario where our approach can provide additional insight not afforded by standard DR approaches.

\vspace{0.20cm}
\noindent\textbf{Revisiting Existing Use of DR.}
\begin{figure}[tb]
    \centering
    \includegraphics[trim={0 0 0 0.80cm}, clip, width=0.95\linewidth]{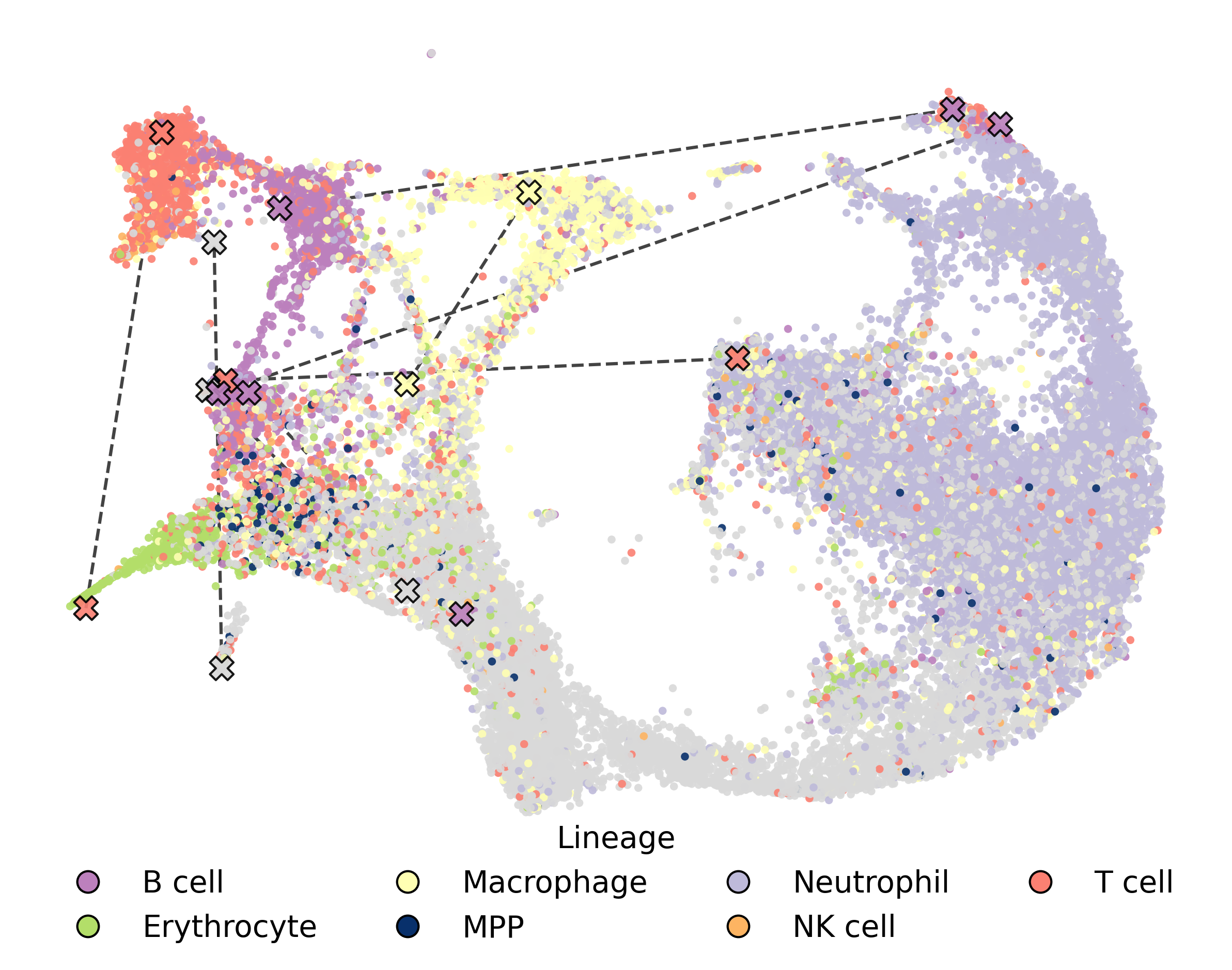}
    \caption{Projection of a mice cell landscape, colored by cell-type as in Becht et al.\cite{bechtDimensionalityReductionVisualizing2019b}. Some ambiguous instances belong in multiple neighborhoods.}
    \label{fig:singlecell}
\end{figure}
DR is a popular approach for analyzing tissue composition in single-cell analysis. Becht et al.~\cite{bechtDimensionalityReductionVisualizing2019b} discuss how UMAP can be used to visualize high-dimensional single-cell data and to reveal meaningful cellular structure through cluster separation and continuity relationships. Transitions between clusters are often interpreted as potential developmental trajectories or intermediate states. Therefore, knowing whether individual cells should be placed in multiple positions may be informative. Conversely, if cells are ambiguous but the embedding would display them only at a fixed position, important biological insights may be missed. Therefore, detecting and representing this ambiguity can improve the interpretation of continuous processes and rare intermediates in single-cell datasets. 

We use the same UMAP parameterization of Becht et al.~\cite{bechtDimensionalityReductionVisualizing2019b} and use their preprocessed subset of 40,000 instances from the Han dataset~\cite{han2018mapping}, which has been reduced from $25,912$ to $30$ dimensions using PCA. Following \cite{bechtDimensionalityReductionVisualizing2019b}, we set the number of neighbors to $30$. We set $\epsilon = 0.9$ to ensure sufficient sparsity without significantly affecting the neighborhood structure (as discussed later) and $\tau_w = 0.1$. We increased $\tau_w$ here compared to the previous examples as $\tau_w = 0.05$ resulted in an isolated point that would not be attracted enough to any neighborhood. ~\Cref{fig:singlecell} shows our results at $r=2$ (average coverage = $0.6\%$), revealing that there are several instances that present partial neighborhood embedding issues if a standard projection is used. While we are not domain experts in single‑cell biology, our goal here is to highlight that acknowledging and visualizing ambiguous instances in DR embeddings may support more nuanced biological interpretation, potentially leading to overlooked insights.

\section{Quantitative Analysis}\label{sec:quant}
\noindent\textbf{Number of Neighbors.} We conduct an experiment to examine how the number of neighbors used to construct the relationship graph affects the results. Using the SVHN dataset with the same parameterization as in the previous examples (including sparsification), we vary only the neighborhood size and report the number of instances split at different radii (see \cref{fig:neighbors}). Note that this differs from the number of instances identified as ambiguous, as additional decision rules determine whether an ambiguous instance is ultimately split. The number of splits is therefore the more informative quantity, as this is what appears in the produced projections. \Cref{fig:neighbors} shows the results. We can see that at only $5$ nearest neighbors, there are instances that are split at $r = 1$ all the way up until $r = 8$ (\cref{fig:neighbors}A). However, once we increase the number of neighbors to $10$, there are no more instances split at radii $5, 6, 7 \text{ and } 8$. Since, with $5$ neighbors, we get a very sparse graph, this is expected. Interestingly, we do see an increase in ambiguous instances being split at $r = 1$ for $10$ neighbors (\cref{fig:neighbors}B); this is likely because, with five neighbors, the connections are so sparse that many instances are connected through an LAP or connected to only one edge. Other than that, the results are as expected; the larger the number of neighbors, the fewer ambiguous instances that get split, and the radius at which they are split decreases. As the number of neighbors increases, the graph becomes more connected, hence there are more paths that may connect nodes, so instances remain ambiguous at a smaller radius. To increase the number of ambiguous instances with a larger number of neighbors, it may be beneficial to sparsify more aggressively. However, this depends on the quality of the projection remaining after sparsification with respect to neighborhood preservation; we investigate this next.
\begin{figure}[tb]
    \centering
    \includegraphics[width=0.85\linewidth]{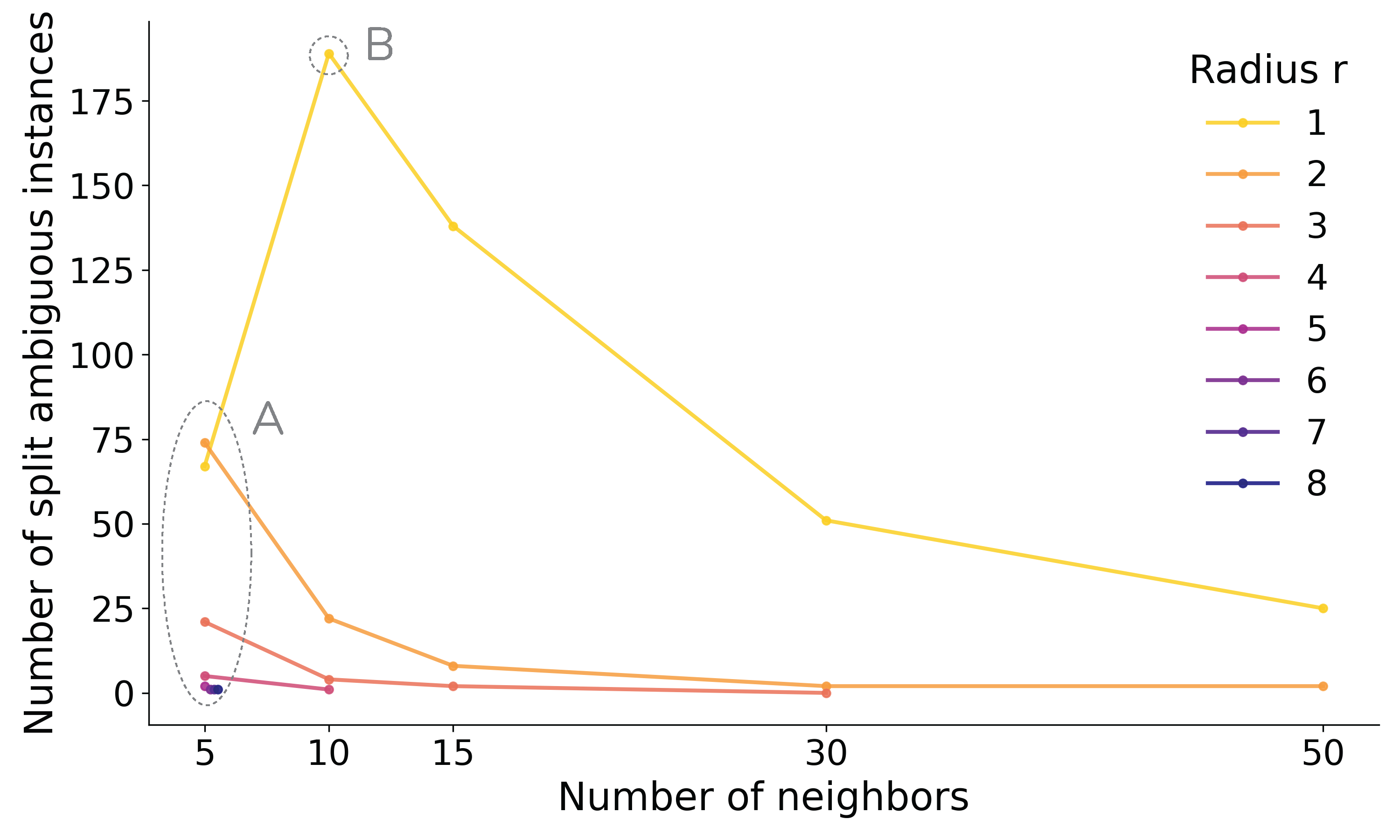}
    \caption{Increasing the number of neighbors mostly decreases the number of instances that are split, as a result of increased connectivity in the graph.}
    \label{fig:neighbors}
\end{figure}

\vspace{0.20cm}
\noindent\textbf{Sparsification.}
To assess whether our graph sparsification preserves the local neighborhood structure required to produce low-dimensional embeddings that are comparable to the ones created using a non-sparsified version, we measure the extent to which embeddings of a sparsified graph retain the same $k$-nearest-neighbor (kNN) relationships as embeddings of the original graph. We quantify this using the fraction of preserved $k$-nearest neighbors between two embeddings, defined as
\begin{equation*}
\mathrm{preservedNN@k}(A, B) =
\frac{1}{n} \sum_{i=1}^n 
\frac{\left|\mathrm{NN}_k^{A}(i)\cap \mathrm{NN}_k^{B}(i)\right|}{k},
\end{equation*} where $A$ and $B$ are the projections and, ${\text{NN}}_k^{A}$ and ${\text{NN}}_k^{B}$ are their $k$-nearest neighbors, respectively. This measure is related to the Jaccard-based neighborhood comparison of Martins et al.~\cite{martinsExplainingNeighborhoodPreservation}, but we normalize by $k$ rather than by the union size, as we aim to assess what percentage of neighbors are retained. We prefer this metric over rank-based neighborhood-preservation measures~\cite{kaski2003trustworthiness, leeQualityAssessmentDimensionality2009a}, which can be overly sensitive by imposing strong penalties for minor rank shifts. Our focus, however, lies primarily on whether neighbors are present (or absent).

Given the non-deterministic nature of UMAP, which can result in distinct layouts for the same input (parametrization + dataset), a direct comparison between a single embedding of an original UMAP graph $G$ and its sparsified version $\bar{G}$ is insufficient. We therefore measure the relative effect of sparsification by comparing its variability against UMAP’s inherent variability. In this process, we first estimate the baseline variability by generating embeddings from the original graph $G$, yielding the set $\mathscr{P}$. We also generate a set of embeddings from $\bar{G}$, which we denote as $\mathscr{Q}$. Then we evaluate the neighborhood overlap ($\mathrm{preservedNN@k}$) between every pair of projections in $\mathscr{P} \times \mathscr{Q}$. To compare the effect of sparsification with the baseline UMAP variability, we form ratios:
\begin{equation*}
\rho(S, \bar{S}) =
\frac{
\bar{S}
}{
S
},
\end{equation*} where $\bar{S}$ is defined as $\mathrm{preservedNN@k}$ between an embedding $A \in \mathscr{P}$ and $B \in \mathscr{Q}$, and $S$ is defined as $\mathrm{preservedNN@k}$ between an embedding $A \in \mathscr{P}$ and $B \in \mathscr{P}$, we call this ratio $\rho_{2D}$. 

We also conduct a similar analysis using the original high-dimensional data $X$ as a reference. For each embedding in $\mathscr{P}$ and $\mathscr{Q}$, we compare them to the high-dimensional space. This means that $\rho(S, \bar{S})$ is defined as $\mathrm{preservedNN@k}$ between an embedding $A \in \mathscr{Q}$ and the high-dimensional space $X$, while $S$ refers to the $\mathrm{preservedNN@k}$ between an embedding $A \in \mathscr{P}$ and $X$. This ratio captures how sparsification affects the alignment between the embedding and the true high-dimensional neighborhoods; we call it $\rho_{HD}$.

We used the datasets from the previous examples with the same parameters described before. We created five embeddings for each graph $G$ and $\bar{G}$, and set $k$ to match the number of neighbors used to construct the graph. We use spectral initialization to generate all projections. If the sparsification retains structure within the same variation as full graphs, the ratios $\rho_{2D}$ will cluster near $1$. Ratios substantially below $1$ indicate that sparsification produces embeddings with more variation in local structures. If the preservation of the high-dimensional space is very similar for the sparsified graphs compared to the full graphs, $\rho_{HD}$ will be near $1$, whereas a low score indicates worse neighborhood preservation than the full graphs. Results are summarized in \Cref{fig:boxplots}. CBR (the text dataset), MNIST, and SVHN datasets have high ratios both for $\rho_{2D}$ and $\rho_{HD}$, so the sparsification does not affect local neighborhood structures. However, the RNA-sequencing (single-cell dataset) shows quite low neighborhood preservation across the embeddings ($\rho_{2D}$). This may be the result of the high intrinsic dimensionality of this dataset, where there are many configurations in which the embedding has a similar UMAP cross-entropy. Another factor may be the larger number of neighbors used for this graph. However, the preservation of nearest neighbors considering the high-dimensional space is high, around $85\%$. This shows that the embeddings from $G$ and $\bar{G}$ preserve the original neighborhood structures similarly, indicating that the divergence between the different layouts lies in the convergence of the UMAP optimization, potentially due to the high intrinsic dimensionality of the RNA-sequencing dataset or the neighborhood size.

\begin{figure}[tbp]
    \centering
    \includegraphics[width=0.72\linewidth]{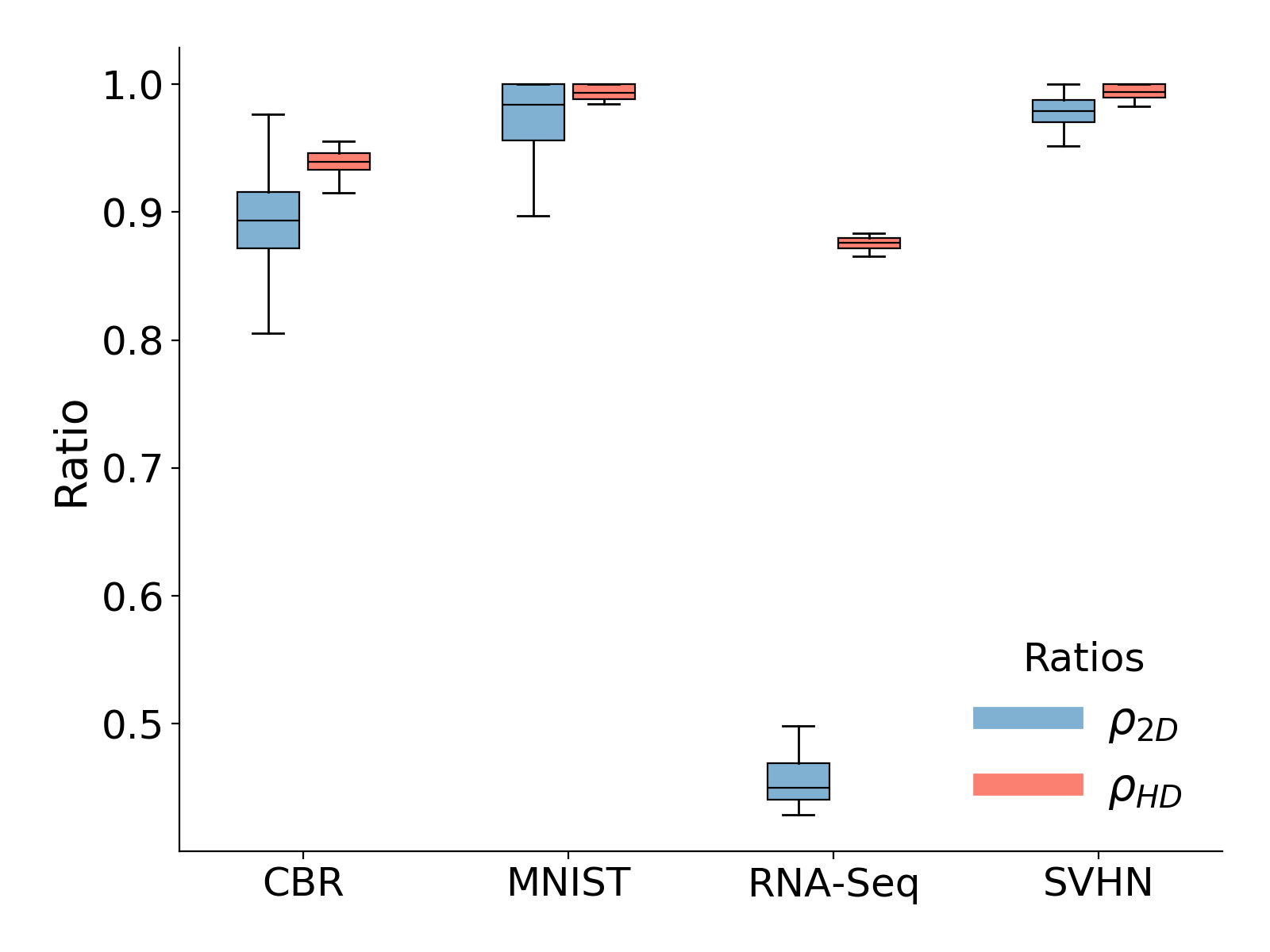}
    \caption{Ratios of $G$ or $\bar{G}$ embeddings vs high dimensional neighborhood preservation.}
    \label{fig:boxplots}
\end{figure}

\vspace{0.20cm}
\noindent\textbf{Standard Metrics Cannot Capture Ambiguity.}
Furthermore, we want to assess whether existing quality metrics can identify ambiguous instances. The most relevant metrics for us are continuity and trustworthiness since they are based on local neighborhoods. We compute these metrics per point for projections generated with the same parameterization as before, but without applying any splits, since these metrics are not designed for multiple points that represent the same instance. \Cref{fig:twcont} shows the results for MNIST, with similar outcomes for the other datasets (see supplemental material). In this figure, all points representing high-dimensional instances in $LAP_r$ $(r=2)$ are colored in light orange. Light blue points are the remaining. We also highlight all points with $|S(v_i)| > 1$ using a triangle marker. From this result, we observe that ambiguous and split instances are spread throughout the distributions for both metrics. Thus, these metrics cannot explicitly identify ambiguous instances.

\begin{figure}[thb]
    \centering
    \includegraphics[width=1.0\linewidth]{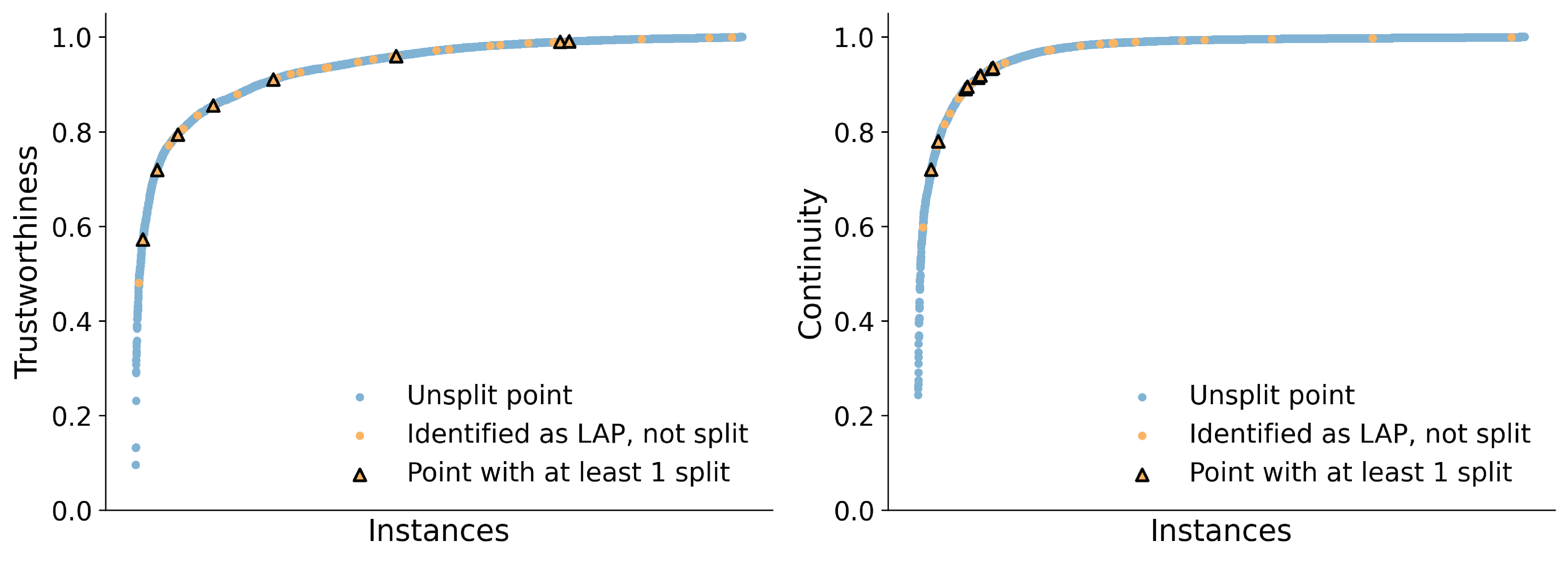}
    \caption{Distribution of trustworthiness and continuity for the examples shown previously. We highlight points representing ambiguous instances (light orange) and those that get split with a triangle marker. Usual DR quality metrics do not capture ambiguous instances.}
    \label{fig:twcont}
\end{figure}


\section{Discussion and Limitations}\label{sec:limitations}
\noindent\textbf{Sparsification. } 
Our approach requires the underlying graph to be relatively sparse, which motivated the introduction of a sparsification step. In addition to the main parameter $\epsilon$, the employed method includes parameters related to the linear solvers used; the full parameterization for the results is listed in the supplemental materials. Although results can vary with different parameter settings, we found the overall behavior to be quite stable with respect to neighborhood consistency. To verify that sparsification does not distort the final embeddings, we applied the quantitative neighborhood-preservation analysis described earlier (see \cref{sec:quant}). We recommend that users of our approach perform the same check. In particular, in our experiments, when very aggressive sparsification was required to obtain any ambiguous splits, the resulting embeddings showed noticeably lower quality, which may indicate that such instances do not naturally occur in the data.  In contrast, in some cases the sparsification step may not be required to identify ambiguous instances if the input graph is sparse (see supplemental material). It may be a direction for future work to automate the parameter selection for the sparsification step based on the neighborhood quality retained.

\vspace{0.20cm}
\noindent\textbf{Ambiguity Definition. } Given our definition, no two ambiguous instances can exist within the same neighborhood. This constraint holds for any LAP in any graph structure, and this is precisely the structure that allows us to find these ambiguous instances. In particular, when multiple instances lie between multiple neighborhoods, they effectively form a neighborhood of their own; since our approach is aimed at identifying individual instances as ambiguous as opposed to entire neighborhoods, we will not identify them as ambiguous. However, alternative approaches, for instance based on community overlap~\cite{10.1145/2433396.2433471}, may identify different types of troublesome instances, suggesting alternative perspectives for future work. In general, while ambiguous instances exist in high-dimensional data, they are not an inherent data property. Rather, the neighborhoods that are constructed depend on the specific technique used to create the relationship graph; using a different technique or parameterization may yield different ambiguous instances. 

Additionally, while our definition excludes true outliers (as ambiguity requires similarity), our detection procedure does not explicitly filter them. In principle, an outlier could appear connected to several neighborhoods depending on graph construction, producing many weak splits that may manifest as visually isolated points due to insufficient attraction. However, we did not observe such cases in our experiments, so analysis of outlier behavior is an open question for future work.

\vspace{0.20cm}
\noindent\textbf{Partial Neighborhood Embedding vs Missing Neighbors.} We define partial neighborhood embedding as an artifact in which an ambiguous instance is placed within only one of its neighborhoods in the visual space. This relates to the well‑known missing neighbors artifact, where a point in the projection is not placed near all of its high‑dimensional neighbors~\cite{martinsVisualAnalysisDimensionality2014, kaski2003trustworthiness}. Partial neighborhood embedding can be seen as a specific subtype of this artifact: while it always involves missing neighbors, not all missing neighbor cases stem from ambiguity; they may also arise from unrelated distortions introduced during mapping. This means that (some) missing neighbor artifacts may be resolved by improving the mapping phase, while our artifact arises from the data. We therefore introduce a distinct name to highlight this more fine‑grained, ambiguity-driven phenomenon. Since current metrics cannot capture this specific artifact (see~\cref{sec:quant}), an opportunity is to develop DR quality metrics that can do so.

\vspace{0.20cm}
\noindent\textbf{Parametrization. } Several parameters, such as the number of neighbors and the distance metric, are external and impact only the construction of the underlying graph. Our approach introduces two remaining parameters: the radius $r$ and the weight-ratio threshold $\tau_w$. In the examples shown in this paper, we used $r=2$ because it is very local yet still requires ambiguity at larger scales. In denser graphs, there are not always splits above $r = 2$. However, when the input graph is already sparse, varying $r$ can be insightful, as it allows us to assess the importance of these instances and gauge the strength of the ambiguity. Instances that are ambiguous (and split) at larger radii connect larger, disjoint neighborhoods than instances at lower radii. We conducted an experiment using the Vispub graph~\cite{Isenberg:2017:VMC}, which contains information about every paper that has appeared at any IEEE VIS venue. In this graph, nodes represent authors, and edges are weighted by the number of papers on which two authors have co‑authored. We applied our approach to this graph, and as it is quite sparse, we skipped the sparsification step. We observed a large number of ambiguous instances at low radii, but as $r$ increased, we progressively isolated highly influential authors whose work spans multiple communities in the field. 

The parameter $\tau_w$ controls how strong the connection between an ambiguous instance and a component (neighborhood) should be to be considered meaningful; increasing $\tau_w$ filters out very weak connections to components and can also reduce spurious splits that appear visually isolated due to insufficient attractive forces, while lower values allow more ambiguous instances in the projection. Another approach is to consider the number of connections to a component; however, this is less suitable in our setting, as our method relies on weighted graphs.

\vspace{0.20cm}
\noindent\textbf{Mapping Phase and Visual Encoding. } The mapping phase should handle graphs with an adjusted structure from vertices being split and their neighbors being distributed. We considered introducing additional edges between split vertices and non-adjacent reachable nodes to account for the lower degree of split vertices, but we did not do so because adding them might create neighborhoods that are less faithful to the original graph structure. In general, projections with many splits will differ from the original projection, but to the best of our knowledge, there is no way to verify these projections using current metrics, since these metrics do not account for ambiguity. This is also a potential direction for future work: developing DR quality metrics that handle instances represented by multiple points. It would also be interesting to investigate whether points that were previously thought to be unstable as a result of the mapping process are actually ambiguous~\cite{jung2025ghostumap2}.

As this study focuses mainly on aspects of ambiguity in the DR process, we opted for simple visual encodings to highlight the points representing ambiguous instances in the embedding. However, further investigation of this represents an important direction for future work.

\section{Conclusion}\label{sec:conclusions}
This work addresses a problem that has received little attention in the field of DR: the presence of \textit{ambiguous instances} in the high-dimensional space and their impact on the produced visual representations, resulting in artifacts we call \textit{partial neighborhood embedding}. Existing DR quality metrics fail to capture this issue because they cannot distinguish between artifacts introduced by the mapping phase and those arising from specific structures present on the data. As a result, data ambiguity and its consequences for DR remain largely unexplored.

Our results show that it is essential to account for ambiguity in DR, especially when interpreting fine-grained local structure, since failing to do so can obscure memberships of multiple neighborhoods. To mitigate this issue, we introduce an approach grounded in established concepts from graph theory, connecting the structural definition of data ambiguity to the resulting misleading visual artifacts. Through a series of examples, we show that our method can reveal additional information not captured by standard DR techniques, and our quantitative analyses further solidify our approach. Therefore, recognizing ambiguity as an important element in interpreting local DR opens up new and interesting opportunities for the analysis and evaluation of high-dimensional data.


\bibliographystyle{abbrv-doi-hyperref}

\bibliography{template}

\appendix 
\crefalias{section}{appendix} 
\end{document}